\documentclass[letterpaper]{article}%
\usepackage{aaai25}%
\usepackage{times}%
\usepackage{helvet}%
\usepackage{courier}%
\usepackage[hyphens]{url}%
\usepackage{graphicx}%
\usepackage{natbib}%
\usepackage{caption}%
\usepackage[linesnumbered,ruled,vlined]{algorithm2e}
\usepackage{color}

\usepackage{amsmath}
\usepackage{amsthm}

\newtheorem{example}{Example}
\newtheorem{theorem}{Theorem}
\newtheorem{definition}{Definition}
\newtheorem{proposition}{Proposition}

\newcommand{\ie}{{\it i.e.}}
\newcommand{\eg}{{\it e.g.}}
\newcommand{\HE}{\textit{HE}}
\newcommand{\LE}{\textit{LE}}
\newcommand{\ar}{\textit{at-robby}}
\newcommand{\f}{\textit{-from}}
\newcommand{\q}{\textit{-to}}
\newcommand{\pre}{\textit{pre}}
\newcommand{\eff}{\textit{eff}}
\newcommand{\add}{\textit{add}}
\newcommand{\del}{\textit{del}}

\newcommand{\at}{\textit{at}}
\newcommand{\white}{\textit{white}}
\newcommand{\black}{\textit{black}}
\newcommand{\free}{\textit{free}}
\newcommand{\carry}{\textit{carry}}
\newcommand{\move}{\textit{move}}
\newcommand{\charge}{\textit{charge}}
\newcommand{\pick}{\textit{pick}}
\newcommand{\drop}{\textit{drop}}

\title{An Automatic Sound and Complete Abstraction Method for Generalized Planning with Baggable Types}
\author{%
    Hao Dong,
    Zheyuan Shi,
    Hemeng Zeng,
    Yongmei Liu\thanks{Corresponding author}
}
\affiliations{%
    Dept. of Computer Science, Sun Yat-sen University, Guangzhou 510006, China\\%
    ymliu@mail.sysu.edu.cn%
}%

\usepackage{bibentry}%

\begin{document}

\maketitle

\begin{abstract}
Generalized planning is concerned with how to find a single plan to solve multiple similar planning instances. Abstractions are widely used for solving generalized planning, and QNP (qualitative numeric planning) is a popular abstract model. Recently, Cui et al. showed that a plan solves a sound and complete abstraction of a generalized planning problem if and only if the refined plan solves the original problem. However, existing work on automatic abstraction for generalized planning can hardly guarantee soundness let alone completeness. In this paper, we propose an automatic sound and complete abstraction method for generalized planning with baggable types. We use a variant of QNP, called bounded QNP (BQNP), where integer variables are increased or decreased by only one. Since BQNP is undecidable, we propose and implement a sound but incomplete solver for BQNP. We present an automatic method to abstract a BQNP problem from a classical planning instance with baggable types. The basic idea for abstraction is to introduce a counter for each bag of indistinguishable tuples of objects. We define a class of domains called proper baggable domains, and show that for such domains, the BQNP problem got by our automatic method is a sound and complete abstraction for a generalized planning problem whose instances share the same bags with the given instance but the sizes of the bags might be different. Thus, the refined plan of a solution to the BQNP problem is a solution to the generalized planning problem. Finally, we implement our abstraction method and experiments on a number of domains demonstrate the promise of our approach. 
\end{abstract}

\begin{links}
\link{Code}{https://github.com/sysulic/ABS}
\end{links}

\section{Introduction}

Generalized planning (g-planning in short), where a single plan works for multiple planning instances, remains a challenging problem in the AI community \cite{Levesque05,Srivastava08AbstractCounting,Hu11MultipleEnvironments,Aguas16GP,Bonet18Features,Illanes19QP,France21LearningFromExample}. Computing general solutions with correctness guarantees is a key problem in g-planning. 

Abstraction methods play an important role in solving g-planning problems. The idea is to abstract a given low-level (LL) problem to get a high-level (HL) problem, solve it and then map the solution back to the original problem. Based on the agent abstraction framework of \citeauthor{BanihashemiGL17} (\citeyear{BanihashemiGL17}), \citeauthor{cui} (\citeyear{cui}) proposed a uniform abstraction framework for g-planning. \citeauthor{cui2023} (\citeyear{cui2023}) proposed an automatic verification method for sound abstractions of g-planning problems.

Qualitative numeric planning (QNP) \cite{Srivastava11QNP}, an extension of classical planning with non-negative real variables that can be increased or decreased by some arbitrary amount, has been a popular abstract model for g-planning. 
A number of QNP solvers have been developed, including FONDASP \cite{FONDASP} and DSET \cite{hemeng22}.
\citeauthor{Bonet18Features} (\citeyear{Bonet18Features}) abstracted a class of g-planning problems into QNP problems. 
In recent years, the automatic generation of  abstractions for g-planning has attracted the attention of researchers. 
\citeauthor{B2019a} (\citeyear{B2019a}) learned a QNP abstraction of a g-planning problem from a sample set of instances, however, the abstraction is only guaranteed to be sound for sample instances. \citeauthor{B2019b} (\citeyear{B2019b}) showed how to obtain a first-order formula that defines a set of instances on which the abstraction is sound. \citeauthor{Illanes19QP} (\citeyear{Illanes19QP}) considered a class of g-planning problems called quantified planning problems based on the idea of quantifying over sets of similar objects, and adapted QNP techniques to produce general solutions. They also proposed to use the work by \citeauthor{Riddle16Baggy} (\citeyear{Riddle16Baggy}) to build a quantified planning problem out of a planning instance. However, they did not address the soundness and completeness issues of their abstraction method. A closely related line of work is  reformulation \cite{Riddle16Baggy,Compiling2016}, where to reduce the state space, a classical planning instance is reformulated by quantifying over indistinguishable objects. 

In this paper, we propose an automatic method to abstract a QNP problem from a  classical planning instance with baggable types.  We use a variant of QNP, called bounded QNP (BQNP), where integer variables are only increased or decreased by one.
The basic idea for abstraction is to introduce a counter for each bag of indistinguishable tuples of objects. 
The reason we use BQNP instead of QNP as our abstract model is that our target abstract actions are those which perform an action on arbitrary elements from bags, thus increasing or decreasing the size of bags by one. 
We resolve the technical complications involved with the definitions of numeric variables, abstract goal, and abstract actions. In particular, we have to ensure the numeric variables are independent from each other, since QNP cannot encode constraints among numeric variables. 
We define a class of domains called proper baggable domains, and 
show that for such domains, the BQNP problem is a sound and complete abstraction for a g-planning problem whose instances share the same bags  with the given instance but the sizes of the bags might be different. 
Since BQNP is undecidable, we propose a sound but incomplete algorithm to test if a BQNP policy terminates, and implement a basic BQNP solver based on the QNP solver DSET. 
Finally, we implement our abstraction method, and experiments on a number of  domains demonstrate its promise. 
To the best of our knowledge, this is the first automatic abstraction method which can guarantee both soundness and completeness.

\section{Preliminaries}

\subsection{Situation Calculus}

The situation calculus \cite{sc2001} is a many-sorted first-order language with some second-order ingredients suitable for describing dynamic worlds. There are three disjoint sorts: $action$ for actions, $situation$ for situations, and $object$ for everything else. The language also has the following components: a situation constant $S_0$ denoting the initial situation; a binary function $do(a, s)$ denoting the successor situation to $s$ resulting from performing action $a$; a binary relation $Poss(a, s)$ indicating that action $a$ is possible in situation $s$; a set of relational (functional) fluents, i.e., predicates (functions) taking a situation term as their last argument. We call a formula with all situation arguments eliminated a situation-suppressed formula $\phi$. We use $\phi[s]$ to denote the formula obtained from $\phi$ by restoring $s$ as the situation arguments to all fluents.

In the situation calculus, a particular domain of application can be specified by a basic action theory (BAT) of the form
$\mathcal{D}=\Sigma\cup\mathcal{D}_{ap}\cup\mathcal{D}_{ss}\cup\mathcal{D}_{una}\cup\mathcal{D}_{S_0}$,
where $\Sigma$ is the set of the foundational axioms for situations, $\mathcal{D}_{ap}$, $\mathcal{D}_{ss}$ and $\mathcal{D}_{una}$ are the sets of action precondition axioms, successor state axioms, unique name axioms for actions, and $\mathcal{D}_{S_0}$ is the initial knowledge base stating facts about $S_0$.

\citeauthor{golog} (\citeyear{golog}) introduced a high-level programming language Golog with the following syntax:
\begin{equation*}
\delta::=\alpha \mid \phi?\mid \delta_1;\delta_2\mid \delta_1|\delta_2 \mid \pi x.\delta \mid \delta^*,
\end{equation*}
where $\alpha$ is an action term; $\phi?$ is a test;  $\delta_1;\delta_2$ is sequential composition; $\delta_1|\delta_2$ is non-deterministic choice; $\pi x.\delta$ is non-deterministic choice of action parameter; and $\delta^*$ is  nondeterministic iteration. The semantics of Golog is defined using an abbreviation $Do(\delta, s, s')$, meaning that executing the program $\delta$ in situation $s$ will result in situation $s'$. 

The counting ability of first-order logic is very limited. \citeauthor{FOCN} (\citeyear{FOCN}) extended FOL by counting, getting a new logic FOCN. The key construct of FOCN are counting terms of the form $\#\overline{y}.\varphi$, meaning the number of tuples $\overline{y}$ satisfying formula $\varphi$.  The 
situation calculus has been extended with counting by, \eg, \citeauthor{Zarrie2016DecidableVO} (\citeyear{Zarrie2016DecidableVO}).  

\subsection{STRIPS}

\begin{definition} \rm A STRIPS domain is a tuple $D=\langle T, P, A\rangle$, where $T$ is a set of object types, $P$ is a set of predicates and $A$ is a set of actions, every $a\in A$ consists of preconditions $\pre(a)$, add list $\add(a)$ and delete list $\del(a)$, where $\pre(a)$ is a formula that must be satisfied before $a$ is executed, $\add(a)$ is a set of the true ground atoms after doing $a$, and $\del(a)$ is a set of the false ground atoms after performing $a$.
\end{definition}

\begin{definition} \rm
A STRIPS planning instance is a tuple $\mathcal{P}=\langle D, O, I, G \rangle$, where $D$ is a STRIPS domain, $O$ is a set of objects of different types, $I$, the initial state, is a set of ground atoms made from predicates in $P$ and objects in $O$, and $G$, the goal condition, is a set of ground atoms.
\end{definition}

Given a STRIPS domain, it is easy to write its  BAT $\mathcal{D}$. We omit the details here. 

\begin{example}[Gripper World] \label{g1} \rm 
The Gripper domain involves a robot with several grippers and a number of balls at different rooms. The robot robby can move between rooms and each gripper may carry one ball a time. The predicates are: $\at(b, r)$ denotes ball $b$ is at room $r$; $\white(b)$ means $b$ is white; $\black(b)$ means $b$ is black; $\carry(b, g)$ denotes gripper $g$ carries $b$; $\free(g)$ denotes $g$ is free; $\HE(g)$ denotes $g$ is high energy; $\LE(g)$ denotes $g$ is low energy; $\ar(r)$ denotes robby is at $r$. The actions are: $\move(r\f,r\q)$ denotes robby moves from one room to another room; $\charge(g)$ denotes charging $g$; $\drop(b,g,r)$ denotes $g$ drops $b$ at $r$;
$\pick(b,g,r)$ denotes $g$ picks $b$ at $r$, where
\begin{itemize}
    \item $\textit{pre}=\{\at(b, r), \free(g), \ar(r), \HE(g)\}$;
    \item $\textit{eff}=\{\carry(b,g), \LE(g), \neg \at(b,r), \neg \free(g), \neg \HE(g)\}$.
\end{itemize}

Below is a planning instance $\mathcal{P}=\langle D, O, I, G\rangle$, where 
\begin{itemize}
    \item $O=\{b_1, b_2, b_3, b_4, b_5, b_6, b_7, b_8, g_1, g_2, r_1, r_2\}$;
    \item $I=\{\at(b_1, r_1), \at(b_2, r_1), \at(b_3, r_1), \at(b_4, r_1), \\ \at(b_5, r_2), \at(b_6, r_2), \at(b_7, r_2), \at(b_8, r_2), \white(b_1), \\ \white(b_2), \white(b_3), \white(b_4), \black(b_5), \black(b_6), \\ \black(b_7), \black(b_8), \free(g_1), \free(g_2), \HE(g_1), \\ \HE(g_2), \ar(r_1)\}$;
    \item $G=\{\at(b_1, r_2), \at(b_2, r_2), \at(b_3, r_2), \at(b_4, r_2), \\ \at(b_5, r_1), \at(b_6, r_1), \at(b_7, r_1), \at(b_8, r_1)\}$.
\end{itemize}
\end{example}

\subsection{Qualitative Numeric Planning (QNP)}

QNP is classical planning extended with numerical variables that can be decreased or increased by arbitrary amount \cite{Srivastava11QNP}. Given a set of non-negative numerical variables $V_N$ and a set of propositional variables $V_B$, $\mathcal{L}$ denotes the class of all consistent sets of literals of the form $N>0$ and $N=0$ for $N\in V_N$, $B$ and $\neg B$ for $B\in V_B$.

\begin{definition} \rm
A QNP problem is a tuple $\mathcal{Q}=\langle V_N, V_B, Init, Goal, Ops \rangle$ where $V_N$ is a set of non-negative numeric variables, $V_B$ is a set of propositional variables, $Ops$ is a set of actions, every $op\in Ops$ has a set of preconditions $\textit{pre}(op)\in\mathcal{L}$, and effects $\textit{eff}(op)$, $Init\in\mathcal{L}$ is the initial state, $Goal\in\mathcal{L}$ is the goal condition.
Propositional effects of $\textit{eff}(op)$ contain literals of the form $B$ and $\neg B$ for $B\in V_B$. Numeric effects of $\textit{eff}(op)$ contain special atoms of the form $inc(N)$ or $dec(N)$ for $N\in V_N $ which increase or decrease $N$ by an arbitrary amount.
\end{definition}

A \textit{qualitative state (qstate)} of $\mathcal{Q}$ is an element of $\mathcal{L}$ in which each variable has a corresponding literal. A \textit{state} of $\mathcal{Q}$  is an assignment of non-negative values to all $N\in V_N$ and of truth values to $B\in V_B$. An instance  of $\mathcal{Q}$ is a numerical planning instance $Q= \langle V_N, V_B, s_0, Goal, Ops\rangle$ which replaces $Init$ with a state $s_0$ satisfying $Init$.

A policy $\pi$  for a QNP problem $\mathcal{Q}$ is a partial mapping from qstates into actions. Given a policy $\pi$, a $\pi$-trajectory is a sequence of states $s_0,s_1,\ldots $ (finite or infinite) s.t. for all $i\geq 0$, $s_{i+1}$ can be resulted from performing $\pi(\bar{s}_i)$ in $s_i$, where $\bar{s}_i$ is the qstate satisfied by $s_i$. 

We omit the definitions that $\pi$    terminates for $\mathcal{Q}$  and $\pi$   solves $\mathcal{Q}$. 
\citeauthor{Srivastava11QNP} (\citeyear{Srivastava11QNP}) introduced a sound and complete algorithm SIEVE, which tests whether a policy $\pi$ for $Q$ terminates. Given $G$, the qstate transition graph induced by $Q$ and $\pi$, SIEVE iteratively removes edges from $G$ until $G$ becomes acyclic or no more edges can be removed. Then $\pi$ terminates iff $G$ is acyclic.

\subsection{Abstraction for Generalized Planning}

\citeauthor{cui} (\citeyear{cui}) proposed a uniform abstraction framework for g-planning, which we adapt to our setting.

\begin{definition} \rm
A g-planning problem is a tuple $\mathcal{G}=\langle \mathcal{D}, G\rangle$, where $\mathcal{D}$ is a BAT and $G$ is a goal condition.
\end{definition}

A solution to a g-planning problem $\mathcal{G}=\langle \mathcal{D}, G\rangle$ is a Golog program $\delta$ s.t. for any model $M$ of $\mathcal{D}$, $\delta$ terminates and achieves the goal. We omit the formal definition here. 

\begin{definition}[refinement mapping]\rm
A function $m$ is a refinement mapping from the HL g-planning problem $\mathcal{G}_h = \langle \mathcal{D}_h, G_h \rangle$ to the LL g-planning problem $\mathcal{G}_l = \langle \mathcal{D}_l, G_l \rangle$ if for each HL action type $A$, $m(A(\vec{x})) = \delta_{A}(\vec{x})$, where $\delta_{A}(\vec{x})$ is a LL program; for each HL relational fluent $P$, $m(P(\vec{x})) = \phi_{P}(\vec{x})$, where $\phi_{P}(\vec{x})$ is a LL situation-suppressed formula; for each HL functional fluent $F$,  $m(F(\vec{x})) = \tau_{F}(\vec{x})$, where $\tau_{F}(\vec{x})$  is a LL term, possibly a counting term.
\end{definition}

For a HL formula $\phi$, $m(\phi)$ denotes the formula resulting from replacing each HL symbol in $\phi$ with its LL definitions. For a HL program $\delta$, $m(\delta)$ is similarly defined. 

\begin{definition}[$m$-isomorphism]\rm
Given a refinement mapping $m$, a situation $s_{h}$ of a HL model $M_{h}$ is $m$-isomorphic to a situation $s_{l}$ in a LL model $M_{l}$, written $s_{h}\sim_{m}s_{l}$, if:
for any HL relational fluent $P$, and variable assignment $v$, we have $M_{h}, v[s/s_{h}] \models P(\vec{x},s)$ iff $M_{l}, v[ s/s_{l}] \models m(P)(\vec{x}, s)$; for any HL functional fluent $f$, variable assignment $v$, we have $M_{h}, v[s/s_{h}] \models f(\vec{x},s)=y$ iff $M_{l}, v[s/s_{l}] \models m(f)(\vec{x},s)=y$.
\end{definition}

\begin{proposition}\label{p1}
Suppose $s_h\sim_m s_l$. Let $\phi$ be a HL situation-suppressed formula. Then $M_h$, $v[s/s_h]\models \phi[s]$ iff $M_l$, $v[s/s_l]\models m(\phi)[s]$. 
\end{proposition}

In the following definition, $\Delta^{M}_S$ denotes all situations of $M$, $S_0^{M}$ stands for the initial situation of $M$.

\begin{definition}[$m$-bisimulation] \rm
A relation $R\subseteq \Delta^{M_{h}}_{S} \times \Delta^{M_{l}}_{S}$ is an $m$-bisimulation relation, if $\langle S^{M_{h}}_{0}, S^{M_{l}}_{0} \rangle \in R$, and $\langle s_{h}, s_{l}\rangle \in R$ implies that: $s_{h}\sim_{m}s_{l}$; for any HL action type $A$, and variable assignment $v$, if there is a situation $s'_{l}$ s.t. $M_{l},v[s/s_{l},s'/s'_{l}]\models Do(m(A(\vec{x})),s,s')$, then there is a situation $s'_{h}$ s.t. $M_{h},v[s/s_{h},s'/s'_{h}]\models Do(A(\vec{x}),s,s')$ and $\langle s'_{h},s'_{l}\rangle \in R$, and vice versa.
\end{definition}

\begin{definition}\rm
$\mathcal{G}_h$ is a sound $m$-abstraction of $\mathcal{G}_l$ if for each model $M_l$ of $\mathcal{G}_l$, there is a model $M_h$ of $\mathcal{G}_h$ s.t. there is an $m$-bisimulation relation $R$ between $M_{h}$ and $M_{l}$, and for any $\langle s_h, s_l\rangle \in R$,  $M_h,v[s_h/s] \models G_h[s]$ iff $M_l,v[s_l/s] \models G_l[s]$.
\end{definition}

\begin{definition} \rm
$\mathcal{G}_h$ is a complete $m$-abstraction of $\mathcal{G}_l$ if for each model $M_h$ of $\mathcal{G}_h$, there is a model $M_l$ of $\mathcal{G}_l$ s.t. there is a $m$-simulation relation $R$ between $M_{h}$ and $M_{l}$, and for any $\langle s_h, s_l\rangle \in R$,  $M_h,v[s_h/s] \models G_h[s]$ iff $M_l,v[s_l/s] \models G_l[s]$.
\end{definition}

\begin{theorem} \label{thm1}
If $\mathcal{G}_h$ is a sound and complete $m$-abstraction of $\mathcal{G}_l$, then $\delta$ solves $\mathcal{G}_h$ iff $m(\delta)$ solves $\mathcal{G}_l$.
\end{theorem}

\section{Bounded QNP}

In this section, we consider a variant of QNP, called bounded QNP (BQNP), where numeric variables are only increased or decreased by one. Since BQNP is undecidable, we propose a sound but incomplete method to test whether a policy for a BQNP problem terminates, based on which, by adapting a characterization of QNP solutions to BQNP, we propose a sound but incomplete method for BQNP solving.  

\begin{definition} \rm
A BQNP problem is a QNP problem where numeric variables take integer values, $inc(N)$ is interpreted as: $N$ is increased by 1, and similarly for $dec(N)$.
\end{definition}

\begin{definition} \rm \label{sob}
Given a BQNP problem $\mathcal{B}$, a policy $\pi$ for $\mathcal{B}$ is a partial mapping from qualitative states to actions. A policy $\pi$ terminates for $\mathcal{B}$ (resp. solves $\mathcal{B}$) if for every instance of $\mathcal{B}$, the only $\pi$-trajectory  started from the initial state is finite (resp. goal-reaching). 
\end{definition}

As noted in \cite{Srivastava11QNP}, BQNP policies can be used to represent arbitrary abacus programs, so BQNP is undecidable. Formal proof is given in \citeauthor{De2002} (\citeyear{De2002}).

\begin{theorem}\label{undecidable}
The decision problem of solution existence for BQNP is undecidable: there is no algorithm to  decide whether a BQNP problem has a solution.
\end{theorem}

We now analyze the relationship between QNP and BQNP. The following results follow from the definitions:  

\begin{proposition}\label{simple}
Let $\mathcal{Q}$ be a QNP problem, and let $\mathcal{B}$ be its corresponding BQNP problem. Then 
\begin{enumerate}
\item If a policy $\pi$ terminates for $\mathcal{Q}$, then it terminates for $\mathcal{B}$.
\item If a policy $\pi$ solves $\mathcal{Q}$, then it solves $\mathcal{B}$.
\end{enumerate}
\end{proposition}

\citeauthor{hemeng22} (\citeyear{hemeng22}) gave a characterization of QNP solutions, which by a similar proof, holds for BQNP:

\begin{proposition}\label{char}
Given the AND/OR graph $G$ induced by a BQNP problem $\mathcal{B}$, a subgraph $G'$ of $G$, representing a policy for  $\mathcal{B}$, is a solution to $\mathcal{B}$ iff $G'$ is closed, terminating, and contains a goal node.
\end{proposition}

\begin{proof} 
By Def. \ref{sob}, the only-if direction is obvious. For the if direction, assume that  there is an instance of $\mathcal{B}$ s.t. the only $G'$-trajectory  started from the initial state terminates at a non-goal node $s$. Since $G'$ is closed, $s$ will be continued with the execution of an action, which contradicts that the trajectory terminates at $s$. 
\end{proof}

However, for Proposition \ref{simple}, the converse of neither (1) nor (2) holds. In the following, we illustrate with an example. 

\begin{example}\rm \label{loop}
Let $\mathcal{Q}=\langle V_N, V_B, Init, Goal, Ops\rangle$, where $V_N=\{X\}$, $V_B=\{A, B\}$, $Init=\{X>0, A, \neg B\}$, $Goal=\{X=0\}$ and $Ops=\{a, b, c\}$, where $\pre(a)=\{X>0, A, B\}$, $\eff(a)=\{dec(X),\neg A\}$, $\pre(b)=\{X>0, \neg A, B\}$, $\eff(b)=\{dec(X),\neg B\}$, $\pre(c)=\{X>0, \neg A, \neg B\}$, $\eff(c)=\{inc(X),A,B\}$.

\begin{figure}
    \centering
    \includegraphics[width=0.4\textwidth]{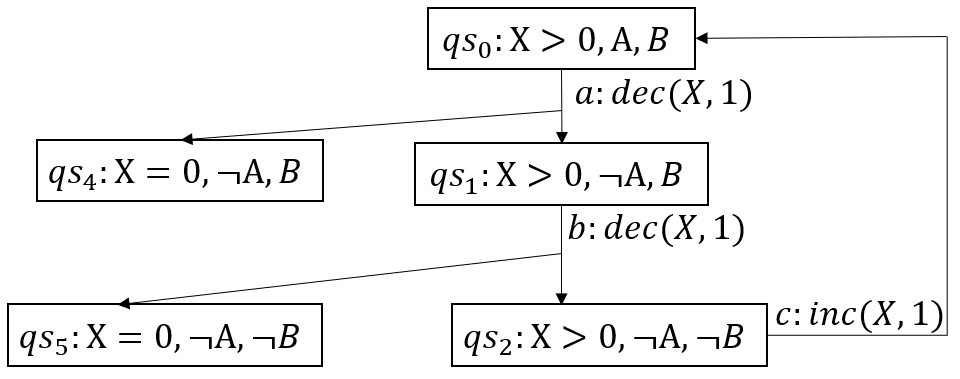}
    \caption{The solution graph of $\mathcal{B}$.}
    \label{fig1}
\end{figure}

Figure \ref{fig1} shows a policy $\pi$ for $\mathcal{Q}$ and the graph induced by $\pi$. By SIEVE, $\pi$ does not terminate for $\mathcal{Q}$, and hence not a solution for $\mathcal{Q}$. However, $\pi$ terminates for $\mathcal{B}$, since there is only one loop, and after each iteration of this loop, $X$ decreases by $1$. By Proposition \ref{char}, $\pi$ is a solution for $\mathcal{B}$. 
\end{example}

Since there are only finitely many policies, by Theorem \ref{undecidable} and Proposition \ref{char}, termination-testing for BQNP policies is undecidable. Motivated by Proposition \ref{simple} and Example \ref{loop}, 
based on SIEVE, we propose a sound but incomplete algorithm (Algorithm \ref{bterm}), to test whether a policy $\pi$ for a BQNP problem  $\mathcal{B}$ terminates.
In the algorithm, a SCC is a strongly connected component, and a simple loop is a loop where no node appears more than once. Given $G$, the qstate transition graph induced by $\pi$, our algorithm first applies SIEVE to $G$ and removes edges. It then returns ``Terminating" if every remaining SCC is a simple loop $g$ where there is a variable $v$ s.t. the number of actions in $g$ that decrease $v$ is more than the number of actions in $g$ that increase $v$. 

\begin{algorithm}
    \caption{Termination-Test}
    \label{bterm}
    \KwIn{$G$, the qstate transition graph induced by a policy $\pi$ for a BQNP problem $\mathcal{B}$}
    \KwOut{``Terminating" or ``Unknown"}
            \BlankLine
            \If{SIEVE($G$) = ``Terminating"}{
                \Return {``Terminating"}\;
            }
           $G' \gets G$ with some edges removed by \textit{SIEVE(}$G$\textit{)}\;
           Compute the SCCs of $G'$\;
            \ForEach{$g \in$ SCCs}{

                \If{$g$ is a simple loop }{
                    choose a variable $v$ s.t. the number of actions in $g$ that decrease $v$ is more than the number of actions in $g$ that increase $v$\;
                    \If{such $v$ exists}{
                        $g$ terminates\;
                    }
                }
             }
            \If{all $g \in$ SCCs terminate}{
              \Return {``Terminating"}\;
            }
            \Return {``Unknown"}\;
\end{algorithm}

\begin{theorem}
Given a BQNP problem $\mathcal{B}$ and a policy $\pi$, let $G$ be the qstate transition graph induced by $\pi$. If Termination-Test$(G)$ returns ``Terminating", then $\pi$ terminates.
\end{theorem}

\begin{proof} By Proposition 2, if SIEVE($G$) returns ``Terminating'',  $\pi$ terminates for $\mathcal{B}$. By soundness of SIEVE, any potential infinite loop resides in $G'$. If a SCC of $G'$ is decided ``terminates'', it cannot be executed infinitely often since the variable $v$ eventually reaches 0 no matter how the other variables behave. When all SCCs of $G'$ terminate, there cannot be any infinite loop in $G'$, and thus $\pi$ terminates. 
\end{proof}

Based on their characterization of QNP solutions, \citeauthor{hemeng22}  (\citeyear{hemeng22}) introduced an approach to solve a QNP by searching for a solution in the induced AND/OR graph, and implemented a QNP solver DSET. By Prop. \ref{char}, a sound but incomplete BQNP solver can be implemented by replacing the termination test in DSET with Alg. \ref{bterm}. 

\citeauthor{termination23} (\citeyear{termination23}) proposed a policy termination test algorithm for the QNP variant with deterministic semantics, where numeric variables are only increased or decreased by a fixed discrete quantity. The algorithm leverages classic results from graph theory involving directed elimination trees and their quotient graphs to compute all ``progress'' variables that change in only one direction (either increasing or decreasing), which are then used to identify all the edges that can be removed. In contrast, our termination test algorithm is specifically tailored for BQNP, is more intuitive and easier to implement.

\section{Our Abstraction Method}

In this section, we show how to abstract a given planning instance $\mathcal{P}$ of a baggable domain into a BQNP problem $\mathcal{B_P}$. The basic idea is to introduce a counter for each bag of indistinguishable tuples of objects.

\subsection{Baggable Domains and Bags}

If two objects can co-occur as the arguments of the same predicate or action, then they can be distinguished by the predicate or action. Thus we first define single types. A baggable type has to be a single type. 

\begin{definition} \rm For a domain $D=\langle T, P, A \rangle$, a type $t \in T$ is single if there is no predicate $p \in P$ or action schema $a \in A$ having more than one type $t$ argument. 
\end{definition}

\begin{definition} \rm Let $t$ be a single type, and $M$ a set of predicates involving $t$, called a predicate group for $t$. 
The mutex group formula of $M$ for $t$, denoted by $\phi_M^t$, is defined as: 
$\forall x. \Sigma_{p\in M} \#\vec{z}.p(x,\vec{z})=1,$
where $x$ is of type $t$. 
\end{definition}

Intuitively, $\phi_M^t$ means: for any object $e$ of type $t$, there is only one predicate $p\in M$ and only one $\vec{u}$ s.t. $p(e,\vec{u})$ holds. 

\begin{definition} \rm\label{siv} Let $D$ be a STRIPS domain. We use $\mathcal{D}$ for its BAT. Let $\Pi$ be a set consisting of a set $\Pi^t$ of predicate groups for each single type $t$. Let $T^S$ denote the set of all single types. We use $\phi_\Pi$ to denote $\bigwedge_{t\in T^S, M\in \Pi^t}\phi_M^t$, \ie, the conjunction of all mutex group formulas. 
We say $\Pi$ is a mutex invariant if $\mathcal{D}_{ap}\cup\mathcal{D}_{ss}\cup\mathcal{D}_{una}\models$ 
\[ \forall s \forall a. \phi_\Pi[s] \land Poss(a,s) \supset \phi_\Pi[do(a,s)].\]
\end{definition}

So $\Pi$ is a mutex invariant means: if $\phi_\Pi$ holds in a state, it continues to hold in any successor state resulting from an executable action. 
If $\Pi$ is a mutex invariant, we call each predicate group in $\Pi^t$ a mutex group for $t$. 

Note that in this paper, we ensure that a mutex group is a state constraint, \ie, holds in any reachable state, by ensuring 1) it holds in the initial states, as will be seen later in the paper; 2) the set of all mutex groups forms a state invariant, as required in the above definition. 

\begin{definition} \label{sw} \rm Let $D$ be a STRIPS domain. A baggable type is a single type $t$ s.t. predicates involving $t$ are partitioned into mutex groups. 
We say that $D$ is a baggable domain if there are baggable types. 
\end{definition}

So if a predicate $p$ involves two baggable types $t_1$ and $t_2$, $p$ must belong to both a mutex group of $t_1$ and a mutex group of $t_2$. Thus true atoms of $p$ induce a 1-1 correspondence between objects of $t_1$ and $t_2$. For Example 1, true atoms of $carry(b,g)$ induce a bijection between balls and grippers. This means each gripper can only carry one ball.

For Example \ref{g1}, types $ball$ and $gripper$ are baggable, but type $room$ is not. The mutex groups of $ball$ are: $M_1=\{\at(b, r), \carry(b, g)\}$ and $M_2=\{\white(b), \black(b)\}$. The mutex groups of $gripper$ are: $M_3=\{\free(g), \carry(b,g)\}$ and $M_4=\{\HE(g), \LE(g)\}$.

In the rest of the section, we assume a baggable domain $D=\langle T, P, A\rangle$ with mutex invariant $\Pi$ and we fix a planning instance $\mathcal{P}=\langle D, O, I, G \rangle$ s.t. $I$ satisfies $\phi_\Pi$. 

We now introduce some notation used throughout this paper. We use $T^B$ to denote the set of baggable types. 
Since a predicate does not contain different arguments of the same baggable type, we use 
$p(T,\vec{y})$ to represent a predicate, where $T\subseteq T^B$
denotes that there is an argument $t$ for each type $t\in T$, 
and $\vec{y}$ stands for arguments of non-baggable types. We also use $p(\vec{x},\vec{y})$ where $\vec{x}$ stands for all arguments of baggable types, and $p(x,\vec{z})$ where $x$ represents an argument of baggable types, and $\vec{z}$ denotes the remaining arguments. 
We use similar notation for action schemas. 
Finally, we use $e$ and $o$ for constants of baggable and non-baggable types, respectively, and $u$ for constants of either type.  

Next, we formalize the concept of bags. Informally, a bag is a set of indistinguishable objects. Essentially, two objects are indistinguishable in a state if they satisfy the same goals and predicates. Thus our formalization of a bag consists of two parts: a subtype of goal-equivalent objects and an extended AVS (attribute value vector). 

\begin{definition} \rm Given goal $G$, we say two objects $e_1$ and $e_2$ of the same baggable type are goal-equivalent if for all predicate $p$ and $\vec{u}$, $p(e_1,\vec{u})\in G$ iff $p(e_2,\vec{u})\in G$. We call each of the equivalence classes of $t$ a subtype of $t$.
\end{definition}

For Example \ref{g1}, the subtype of $gripper$: $st_1=\{g_1, g_2\}$. For $ball$: $st_2=\{b_1,b_2,b_3,b_4\}$ and $st_3=\{b_5,b_6,b_7,b_8\}$.

We now use mutex groups to define attributes of objects. We first explain the intuitive idea. The basic way to define attributes of objects of a type $t$ is to use each predicate involving $t$ as an attribute, and {\em true} and {\em false} as attribute values. However, this can be improved for baggable types. Note that for a baggable type $t$, predicates involving $t$ are partitioned into mutex groups, and for any object of type $t$, at any reachable state, one and only one predicate from the group holds. Thus we can use each mutex group as an attribute, and elements of the group as attribute values.

\begin{definition} \rm Let $t$ be a baggable type. We call an $M\in \Pi^t$ an attribute of objects of type $t$.  Let $p(T,\vec{y})\in M$ where $t\in T$. Let $\vec{o}$ be an instantiation of $\vec{y}$. We call $av(T)=p(T,\vec{o})$ an attribute value for $M$, where $av(T)$ denotes that $T$ is the set of variables for $av$. We use $D_M$ to denote the set of attribute values for $M$.
\end{definition}

For Example \ref{g1}, $D_{M_1}=\{\at(b, r_1), \at(b, r_2), \carry(b, g)\}$.

\begin{definition} \rm Let $t\in T^B, \Pi^t=\{M_1,\ldots,M_m\}$ and $avs(T)=(av_1(T_1),\ldots,av_m(T_m)) \in \times_{i=1}^m D_{M_i}$ where $T=\bigcup_i^m T_i$. We call $avs(T)$ an attribute value vector (AVS) for $t$. $Avs^t$ denotes the set of all attribute value vectors for $t$.
\end{definition}

For Example \ref{g1}, for type $ball$, $Avs^b=\{\at(b, r_1)\wedge \white(b), \at(b, r_1)\wedge \black(b), \at(b, r_2)\wedge \white(b), \at(b, r_2)\wedge \black(b), \carry(b, g) \wedge \white(b), \carry(b, g)\wedge \black(b)\}$. For type $gripper$, $Avs^g=\{\free(g)\wedge \HE(g), \free(g)\wedge \LE(g), \carry(b, g)\wedge \HE(g), \carry(b, g)\wedge \LE(g)\}$.

Our initial idea is to introduce a counter for each AVS. Then for Example \ref{g1}, we have the following counters: 
\begin{itemize}
\item $N_1=\#b.\exists g. \carry(b,g) \wedge \white(b)$, 
\item $N_2=\#b.\exists g. \carry(b,g) \wedge \black(b)$, 
\item $N_3=\#g.\exists b. \carry(b,g) \wedge \HE(g)$,
\item $N_4=\#g.\exists b. \carry(b,g) \wedge \LE(g)$. 
\end{itemize}
Since each gripper only carry one ball a time, there would be a constraint $N_1+N_2=N_3+N_4$. However, QNP cannot encode such numeric constraints. To resolve this issue, we define the concept of extended AVSes, and introduce a counter for each extended AVS. Thus instead, we have the following 4 counters, which are independent from each other: 
\begin{itemize}
\item $N_5=\#(b,g).\carry(b,g) \wedge \white(b)\wedge \HE(g)$,
\item $N_6=\#(b,g).\carry(b,g) \wedge \white(b)\wedge \LE(g)$,
\item $N_7=\#(b,g).\carry(b,g) \wedge \black(b)\wedge \HE(g)$,
\item $N_8=\#(b,g).\carry(b,g) \wedge \black(b)\wedge \LE(g)$.
\end{itemize}

The intuitive idea for defining extended AVSes is this. In the above example, true atoms of $carry(b,g)$ induce a 1-1 correspondence between objects of $ball$ and $gripper$. So we should join the AVS $\carry(b,g) \wedge \white(b)$ with one of the AVSes $\carry(b,g) \wedge \HE(g)$ and 
$\carry(b,g) \wedge \LE(g)$, getting $cavs$, and count the pairs $(b,g)$. It might be the case that a gripper is connected to an object of another type $t$ by a binary predicate $p(g,t)$. So we have to further join $cavs$ with an AVS of type $t$, and count the triples  $(b,g,t)$. We continue this process until no further join is possible.

\begin{definition} \rm Let $T$ be a set of baggable types. For each $t\in T$, let $avs^t(T_t) \in Avs^t$. Let $cavs=\bigwedge_{t\in T} avs^t(T_t)$. We call $cavs$ a conjunctive AVS. 
The underlying graph for $cavs$ is a graph whose set of nodes is $T$ and there is an edge between $t_1$ and $t_2$ if $avs^{t_1}\cap avs^{t_2}\neq \emptyset$. We call $cavs$ connected if its underlying graph is connected. 
We call $cavs$ an extended AVS if it is a maximal connected conjunctive AVS.  
We denote the set of extended AVSes with $Eavs$. For a type $t$, we use $Eavs^t$ to denote the set of extended AVSes that extends an AVS for $t$.
\end{definition}

For Example \ref{g1}, let $avs^b=\carry(b, g) \wedge \white(b)$, $avs^g=\carry(b, g)\wedge \HE(g)$. There is $eavs(b,g)=\carry(b, g) \wedge \white(b)\wedge \HE(g)$ and it is maximal. 

We can now finalize our formalization of a bag.

For $T\subseteq T^B$, we use $sts(T)$ to represent a subtype assignment, which maps each $t\in T$ to a subtype of $t$. We also use $sts(\vec{x})$ for $sts(T)$, meaning $x_i$ is of subtype $st_i$. 

A bag is a set of tuples of objects of $T$ satisfying both a subtype assignment $sts(T)$ and an extended AVS $eavs(T)$. 

\subsection{Abstraction Method}

First, a numeric variable counts the size of a bag.

\begin{definition} 
[Numeric variables] \rm  \label{nud} \hspace*{1cm}\\
$
V_N=\{ (sts(T), eavs(T)) \mid sts(T)$ is a subtype assignment,  $eavs(T)\in Eavs\}.
$\\
The refinement mapping is defined as follows: 

$m(sts(T), eavs(T))=\#\vec{x}. \bigwedge_i st_i(x_i) \land eavs(\vec{x})$.
\end{definition}
Recall $\#\vec{x}.\varphi$ means the number of tuples $\vec{x}$ satisfying $\varphi$. 

For Example \ref{g1}, here are some numerical variables that will be used later: $N_1=(st_3(b), \at(b, r_1)\wedge \white(b))$, $N_2=(st_3(b), \at(b, r_1)\wedge \black(b))$, $N_3=(st_3(b), \at(b, r_2)\wedge \white(b))$, $N_{4}=(st_3(b), \at(b, r_2)\wedge \black(b))$, $N_{5}=(st_1(g)\wedge st_3(b), \carry(b, g)\wedge \white(b)\wedge \HE(g))$, $N_{6}=(st_1(g)\wedge st_3(b), \carry(b, g)\wedge \black(b)\wedge \HE(g))$, $N_{7}=(st_1(g)\wedge st_3(b), \carry(b, g)\wedge \white(b)\wedge \LE(g))$, $N_{8}=(st_1(g)\wedge st_3(b), \carry(b, g)\wedge \black(b)\wedge \LE(g))$,
$N_9=(st_1(g), \free(g)\wedge \HE(g))$, $N_{10}=(st_1(g), \free(g)\wedge \LE(g))$.

\begin{definition} 
[Propositional variables] \rm \hspace*{0.5cm}\\
\(V_B=\{ p(\vec{o})\mid p\in P^N, \vec{o} \in O\}, \hspace*{1cm} m(p(\vec{o}))=p(\vec{o})\),  where $P^N$ is the set of predicates s.t. all arguments are of non-baggable types.
\end{definition}

For Example \ref{g1}, $B=\ar(r_1)\in V_B$.

The abstract initial state is simply the quantitative evaluation of the LL initial state. 
\begin{definition}\rm
Abstract initial state $I_h$: 
\begin{itemize}
    \item Propositional variables: For $B=p(\vec{o})\in V_B$, if $p(\vec{o})\in I$, then $B\in I_h$, otherwise  $\neg B\in I_h$;
    \item Numeric variables: For $N=(sts(T),eavs(T))\in V_N$, if $I\models m(N)>0$, then $N>0\in I_h$, else $N=0\in I_h$. 
\end{itemize}
\end{definition}

For Example \ref{g1}, $B=\ar(r_1), N_4>0\in I_h$. 

We now define the abstract goal $G_h$, which is characterized by those numeric variables that are equal to $0$. This is because the goal condition is a partial state that cannot definitively determine which numeric variables are greater than $0$. Thus, we introduce the following two sets: one is the set of all numeric variables involved in $G$, and the other is the set of numeric variables that may be greater than $0$ in $G$.

For a subtype $st$ of baggable type $t$, we define the set of numeric variables associated with $st$ as $V_{st}=\{(sts(T),eavs(T)) \mid st(t)\in sts(T), eavs(T) \in Eavs^t\}$. Let $g=p(\vec{e},\vec{o})\in G$ for predicate $p(T,\vec{y})$, $t\in T$, and $st$ is a subtype of $t$ containing $e_t$, which is the argument of type $t$ from $\vec{e}$. We define a subset of $V_{st}$ as follows: $V_{st}^{g}=\{(sts(T'),eavs(T')) \mid p(T,\vec{o}) \in eavs(T')\}$. 

\begin{definition} \rm
Abstract goal $G_h$: 
Propositional variables: For $B=p(\vec{o})\in V_B$, if $p(\vec{o}) \in G$, then $B\in G_h$;
Numeric variables: For $g=p(\vec{e},\vec{o})\in G$ and $N\in V_{st}\setminus V_{st}^{g}$, we have $N=0\in G_h$. 
\end{definition}  

For Example \ref{g1}, $N_4=0\in G_h$.

We now define abstract actions. Let $a(T,\vec{y})$ be an LL action, and let $N_i(T_i)$, $i=1, \ldots,k$ be numeric variables. We say that these numeric variables are suitable for $a(T,\vec{y})$ if $T_1, \ldots, T_k$ form a partition of $T$. 

\begin{definition}\rm
Let $N_i=(sts_i(T_i),eavs_i(T_i))$, $i=1, \ldots,k$ be suitable for $a(T,\vec{y})\in A$. Let $\vec{o}$ be an instantiation of $\vec{y}$. Let $a_o=a(T,\vec{o})$. If $\bigwedge_i eavs_i(T_i)\models \textit{pre}(a_o)-\{B\in \textit{pre}(a_o)\}$, there is a HL action, denoted by $\alpha=a(\vec{N},\vec{o})$ s.t.
\begin{itemize}
\item $\textit{pre}(\alpha)=\{B\in \textit{pre}(a_o)\} \cup \{N_i>0 \mid N_i\in \vec{N}\}$;
\item $\textit{eff}(\alpha)$ consist of: 1. $l\in \textit{eff}(a_o)$, where $l$ is a propositional literal; 2. for any $i$, if $eavs_i(T_i) \land \textit{eff}(a_o)\land \phi_\Pi$ is inconsistent, then $dec(N_i,1)$; 3. for any $N'=(sts(T'), eavs(T'))\in V_N-\{\vec{N}\}$ s.t. $sts(T')\subseteq \bigwedge_i sts_i(T_i)$ and $\bigwedge_i(eavs_i(T_i)- del(a_o))\wedge add(a_o))\models eavs(T')$, then $inc(N',1)$. 
\item $m(\alpha)=\pi T. \bigwedge_ists_i(T_i)\wedge eavs_i(T_i)\wedge \textit{pre}(a_0)?;a_o$.
\end{itemize} 
\end{definition}

For Example \ref{g1}, $N_1, N_9, N_{10}$ are suitable for $\pick(b,g,r)$.%
There exists a HL action $\alpha=\pick(N_1,N_9,N_{10},r_1)$, with 
$\textit{pre}(\alpha)=\{\ar(r_1), N_1>0, N_9>0\}$, and 
$\textit{eff}(\alpha)=\{dec(N_1,1), dec(N_9,1), inc(N_{10},1)\}$.

Finally, we give a simple example to demonstrate the form of the solutions to the BQNP problems as abstracted by our abstraction method. Suppose there are 3 rooms and some balls, all initially in room S, with some needed to be moved to room A and others to room B. The only available LL  action is ``push(ball, from, to)", which moves a ball directly from one room to another. After abstraction, we obtain a BQNP problem with 6 numeric variables in the form of $N_C^T$, representing the number of balls currently in room C that are intended for room T. A (compact) solution to this BQNP problem is: [{$N_S^A>0$}: push($N_S^A$, S, A), {$N_S^B>0$}: push($N_S^B$, S, B)], which is refined to a LL solution, meaning that when there are balls in room S intended for A (resp. B), we select any such ball and perform ``push(ball, S, A)" (resp. ``push(ball, S, B)". This solution can be used to solve any LL problem where all balls start in room S and need to be moved to either room A or B.

\section{Soundness and Completeness}

In this section, we define a class of baggable domains called proper baggable domains, and prove that our abstraction method is sound and complete for such domains. In particular, 
given an instance $\mathcal{P} $ of a proper baggable domain, 
we define the low-level generalized planning problem $\mathcal{G_P}$, and show that $\mathcal{B_P}$ derived from our abstraction method is a sound and complete abstraction of $\mathcal{G_P}$.

We begin with some propositions which serve to prove the correctness of the abstract goal (Proposition \ref{Giff}). 

For $t\in T^B$, we define $GM^t=\{sts(T\setminus\{t\})\land eavs(T) \mid eavs(T) \in Eavs^t\}$.

\begin{proposition}\label{gmg} $GM^t$ forms a general mutex group, \ie, \\
\(
\phi_\Pi \models \forall x. \Sigma_{\psi\in GM^t} \#\vec{z}.\psi(x,\vec{z})=1.\)
\end{proposition}

\begin{proof}
First, it is easy to see that the set of subtypes of a type $t$ forms a mutex group. Now we define two notions concerning mutex groups. Let $M$ and $M'$ be two mutex groups, either for the same type or for different types. Let $p\in M$, and $M''$ be the set resulting from replacing $p\in M$ by elements from $p\wedge M'=\{p\wedge q \mid q \in M'\}$. It is easy to see that $M''$ is a general mutex group, and we say that $M''$ is obtained from $M$ by refining $p$ with $M'$. Let $M^*$ be obtained from $M$ by refining each of its elements with $M'$. So $M^*$ is also a general mutex group, we call it the joining of $M$ and $M'$.

Since $Avs^t$ is obtained by joining mutex groups for $t$, it is a general mutex group. Similarly, the set of subtype assignments for $T\setminus\{t\}$, denoted $STS(T\setminus\{t\})$, is also a general mutex group. Further, $Eavs^t$ is obtained from $Avs^t$ by refining some of its elements, thus it is also a general mutex group. Finally, by joining $STS(T\setminus\{t\})$ with $Eavs^t$, we get $GM^t$, which is a general mutex group. 
\end{proof}

We use $|st|$ to denote the size of $st$. By proposition \ref{gmg}, for any object $e$ of $st$, there is one and only one $(sts(T),eavs(T))\in V_{st}$ and only one $\vec{u}$ s.t. $sts(e,\vec{u})\land eavs(e,\vec{u})$ holds. Thus we have

\begin{proposition}\label{sum} 
$\phi_\Pi \models |st|=\Sigma_{N\in V_{st}}m(N)$.
\end{proposition}

\begin{proof}
By Proposition \ref{gmg}, for each object $e$ of $st$, there is one and only one $N=(sts(T),eavs(T))\in V_{st}$ and only one $\vec{u}$  s.t. $sts(e,\vec{u})\land eavs(e,\vec{u})$ holds. 
\end{proof}

For Example \ref{g1}, $\Sigma_{i=1}^8 N_i=|st_3|=4$.

\begin{proposition}\label{gsum}
\(\phi_\Pi\cup G \models |st|=\Sigma\{m(N)\mid N\in V_{st}^{g}\}.\)
\end{proposition}

\begin{proof}
For any $e'\in st$, let $g'$ be the ground atom derived from $p(\vec{e},\vec{o})$ by replacing $e_t$ with $e'$, since $e'$ is goal-equivalent to $e_t$, we know $g'\in G$. Now given $G$, we know $g'$ holds. So we can replace $V_{st}$ with $V_{st}^{g}$ in Proposition \ref{sum}.
\end{proof}

\begin{proposition} \label{Giff}
Let $s_h \sim_m s_l$ and $s_l \models \phi_\Pi$. Then $s_h \models G_h$ iff $s_l\models G$.
\end{proposition}

\begin{proof}
($\Rightarrow$) We prove that for all $p(\vec{u})\in G$, $p(\vec{u}) \in s_l$. If $p(\vec{u})\in V_B$, then it is in $G_h$. Since $ s_h\models G_h$, $p(\vec{u})\in s_h$. Since $ s_h \sim_m s_l$, $p(\vec{u})\in s_l$. Now let $g=p(\vec{u})\not\in V_B$. Then, for all $N\in V_{st}\setminus V_{st}^{g}$, we have $N=0$ in $G_h$. By $s_h \models G_h$, $N=0$ in $s_h$. Since $s_h \sim_m s_l$, $s_l\models m(N)=0$. Since $s_l \models \phi_\Pi$, by Propositions \ref{sum} and \ref{gsum}, there is $N'\in V_{st}^{g}$ s.t. $s_l\models m(N')>0$. So $p(\vec{u}) \in s_l$.

($\Leftarrow$) We prove that for all $B\in G_h$, $B\in s_h$ and for all $N=0\in G_h$, $N=0\in s_h$. For any $B\in G_h$, let $m(B)=p(\vec{o})$, then $p(\vec{o})\in G$. Since $s_l\models G$, $p(\vec{o})\in s_l$. Since $s_h\sim_m s_l$, $B\in s_h$. For any $N=0$ in $G_h$, there is $g=p(\vec{e},\vec{o})\in G$ and $N\in V_{st}\setminus V_{st}^{g}$. Since $s_l \models \phi_\Pi \cup G$, by Propositions \ref{sum} and \ref{gsum}, $s_l \models m(N)=0$. By $s_h\sim_m s_l$, $N=0\in s_h$.
\end{proof}

We now define a property of action schemas which ensures that doing any LL ground action changes the value of any numeric variable by at most one.  This will enable us to prove the correctness of abstract actions. 

To see the intuitive idea behind our definition, consider the extended AVS $\carry(b, g)\wedge \white(b)\wedge \HE(g)$. Suppose an action $a(b,g)$ changes the truth values of both $\white(b)$ and  $\HE(g)$. Then we require that it should also change the truth value of $\carry(b, g)$, which connects $b$ and $g$. So the action is atomic in the sense that it cannot be decomposed into an action $a_1(b)$ and an action $a_2(g)$.

\begin{definition}  \label{atomic} \rm  We say that an action schema $a$ is atomic, if
for any $eavs(T)\in Eavs$,  if $a$ changes the value of $p_1(T_1)$ and $p_2(T_2)$ from $eavs(T)$ where $T_1\cap T_2 = \emptyset$, then it also changes the value of some $p_3(T_3)$ from $eavs(T)$ where $T_1\cap T_3 \neq \emptyset $ and $T_2\cap T_3 \neq \emptyset $.
\end{definition}

We now give simple sufficient conditions for atomic actions. If an action just involves one baggable type, then it is atomic. Now consider an action $a(t_1,t_2)$ involving only two baggable types. It is atomic if for any  $p_1(t_1)$ and $p_2(t_2)$ s.t. $a$ changes the values of both predicates, $a$ also changes the value of $p(t_1,t_2)$, for any $p(t_1,t_2)$ s.t. it belongs to the same mutex group as neither $p_1(t_1)$ nor $p_2(t_2)$.

\begin{definition} \rm  
We say that a baggable domain is proper if each action schema is atomic. 
\end{definition}

For Example 1, the only two actions with two baggable types $pick(b,g,r)$ and $drop(b,g,r)$ change the truth value of $carry(b,g)$, which is the only predicate with two baggable types. Thus the domain is proper.

\begin{proposition}\label{changeone} 
Let $N=(sts(\vec{x}), eavs(\vec{x}))\in V_N$. For any atomic action $a$, if $a$ changes the value of $eavs(\vec{e})$, then for any $\vec{e}\thinspace'\neq \vec{e}$, $a$ does not change the value of $eavs(\vec{e}\thinspace')$. 
\end{proposition}

\begin{proof}
Suppose an action $a$ changes the truth values of both $eavs(\vec{e})$ and $eavs(\vec{e}\thinspace')$. We first show that there must exist $p_1(T_1)$ and $p_2(T_2)$ from $eavs(T)$ s.t. $T_1\cap T_2 \neq \emptyset$ and $a$ changes the truth values of $p_1(\vec{e}_{T_1})$ and $p_2(\vec{e}\thinspace'_{T_2})$, where $\vec{e}_{T_1}$ denotes the restriction of $\vec{e}$ to ${T_1}$. Since $a$ changes the values of both $eavs(\vec{e})$ and $eavs(\vec{e}\thinspace')$, there must exist $p_1(T_1)$ and $p_2(T_2)$ from $eavs(T)$ s.t. $a$ changes the truth values of $p_1(\vec{e}_{T_1})$ and $p_2(\vec{e}\thinspace'_{T_2})$. If $T_1\cap T_2 \neq \emptyset$, we are done. So suppose $T_1\cap T_2 = \emptyset$. By Definition \ref{atomic},  there is $p_3(T_3)$ from $eavs(T)$ s.t. $T_1\cap T_3 \neq \emptyset $,  $T_2\cap T_3 \neq \emptyset $, and $a$ changes the value of $p_3(\vec{e}_{T_4},\vec{e}\thinspace'_{T_5})$, where $T_4=T_1\cap T_3$, and $T_5=T_3-T_1$. We show that $\vec{e}_{T_4}=\vec{e}\thinspace'_{T_4}$, thus 
$a$ changes the value of $p_3(\vec{e}\thinspace'_{T_4},\vec{e}\thinspace'_{T_5})$, \ie, $p_3(\vec{e}\thinspace'_{T_3})$.
Assume that $\vec{e}_{T_4}\neq \vec{e}\thinspace'_{T_4}$. Then $a$ does not change the value of  $p_3(\vec{e}\thinspace'_{T_4},\vec{e}\thinspace'_{T_5})$. If it stays false, then $eavs(\vec{e}\thinspace')$ stays false, a contradiction. If it stays true, since $p_3(\vec{e}_{T_4},\vec{e}\thinspace'_{T_5})$ changes, by mutex group, we have $\vec{e}_{T_4}= \vec{e}\thinspace'_{T_4}$, also a contradiction. 

Now let $t\in T_1\cap T_2$, we must have $\vec{e}_{t}=\vec{e}\thinspace'_{t}$, since $a$ can only have one argument of type $t$. There are two cases. 1) $p_1(\vec{e}_{T_1})$ and $p_2(\vec{e}\thinspace'_{T_2})$ are both false before (after resp.)  the change, then $eavs(\vec{e})$ and $eavs(\vec{e}\thinspace')$ are both true after (before resp.) the change. Since $Eavs^t$ is a general mutex group, from $\vec{e}_{t}=\vec{e}\thinspace'_{t}$, we get $\vec{e}=\vec{e}\thinspace'$. 2) $p_1(\vec{e}_{T_1})$ changes from true to false, but $p_2(\vec{e}\thinspace'_{T_2})$ changes from false to true (the symmetric case is similar). Then we must have $eavs(\vec{e}\thinspace')$ changes from false to true. Assume $\vec{e}_{T_1}\neq\vec{e}\thinspace'_{T_1}$. Then we must have $p_1(\vec{e}\thinspace'_{T_1})$ is false before the action since $p_1$ belongs to a mutex group, and remains false since $\vec{e}\thinspace'_{T_1}$ is not contained $a$'s arguments. This contradicts that $eavs(\vec{e}\thinspace')$ changes from false to true. So $\vec{e}_{T_1}=\vec{e}\thinspace'_{T_1}$. $p_1(\vec{e}\thinspace'_{T_1})$ also changes from true to false. Thus $eavs(\vec{e}\thinspace')$ remains false after the action, also a contradiction. So the second case is impossible. 
\end{proof}%
Here $pre(m(\alpha))$ means the precondition for executing the Golog program $m(\alpha)$.

\begin{proposition} \label{eqa} For any $\alpha$, $\textit{pre}(m(\alpha))\Leftrightarrow m(\textit{pre}(\alpha))$.
\end{proposition}

\begin{proof}
Let $\alpha=a(\vec{N},\vec{o})$. Then $\textit{pre}(\alpha)=\bigwedge_i N_i>0\wedge \bigwedge\{B\in \textit{pre}(a(T,\vec{o}))\}$. $\textit{pre}(m(\alpha))=\exists \vec{x}. \bigwedge_i sts_i(\vec{x}_i)\wedge eavs(\vec{x}_i)\wedge \textit{pre}(a(T,\vec{o})) $. Note that $\bigwedge_i eavs(\vec{x}_i) \models \textit{pre}(a(T,\vec{o}))-\{B\in \textit{pre}(a(T,\vec{o}))\}$. Thus both $\textit{pre}(m(\alpha))$ and $m(\textit{pre}(\alpha))$ are equivalent to $\exists \vec{x}. \bigwedge_i sts_i(\vec{x}_i)\wedge eavs(\vec{x}_i)\wedge \bigwedge\{B\in \textit{pre}(a(T,\vec{o}))\}$.
\end{proof}%

\begin{proposition}\label{tg} Let $s_h \sim_m s_l$. For any abstract action $\alpha=a(\vec{N},\vec{o})$ where $a$ is atomic, if $s_h$ leads to $s'_h$ by execution of $\alpha$, then $s_l$ leads to $s'_l$ by execution of $m(\alpha)$ s.t. $s'_h \sim_m s'_l$, and vice versa.
\end{proposition}

\begin{proof} First, $s_h \models \textit{pre}(\alpha)$ iff $s_l \models m(\textit{pre}(\alpha))$, by Proposition \ref{eqa}, iff $s_l \models \textit{pre}(m(\alpha))$.  Let $\vec{e}$ be any instantiation of $\vec{x}$ satisfying $\bigwedge_ists_i(\vec{x}_i)\wedge eavs_i(\vec{x}_i)$. Let $s'_h$ result from $s_h$ by execution of $\alpha$, $s'_l$ result from $s_l$ by execution of $a(\vec{e},\vec{o})$. We show that $s'_h \sim_m s'_l$. For Boolean variables, $\alpha$ has the same effect on them as $a(\vec{e},\vec{o})$ does. For numeric variables, there are three cases. a) $\alpha$ has effect $dec(N_i,1)$. Then $eavs_i(\vec{e}_i)$ changes from truth to false. By Proposition \ref{changeone}, $m(N_i)$ is decreased by one from $s_l$ to $s'_l$. b) $\alpha$ has effect $inc(N,1)$. Then $eavs(\vec{e})$ changes from false to truth. By Proposition \ref{changeone}, $m(N_i)$ is increased by one from $s_l$ to $s'_l$. c) $\alpha$ has no effect on $N$. Then $a(\vec{e},\vec{o})$ does not affect the value of any predicate in $eavs(\vec{x})$. So $v(m(N))$ does not change from $s_l$ to $s'_l$.
\end{proof}

Given an instance $\mathcal{P} $ of a proper baggable domain, we have defined its BQNP abstraction $\mathcal{B_P}$. We now define the LL g-planning problem $\mathcal{G_P}$,
and show that $\mathcal{B_P}$ is a sound and complete abstraction of $\mathcal{G_P}$.
Intuitively, any instance of $\mathcal{G_P}$ shares the same BQNP abstraction with $\mathcal{P}$. 

\begin{definition}\rm Given an instance $\mathcal{P}=\langle D, O, I, G \rangle$, where $D=\langle T, P, A\rangle$, of a proper baggable domain. Let $\mathcal{B_P}$ with initial state $ I_h$ and goal $G_h$ be its BQNP abstraction. 
The LL g-planning problem for $\mathcal{P}$ is a tuple $\mathcal{G_P}=\langle D', O_n, I_l, G_l\rangle$, where 
$D'=\langle T', P, A\rangle$ and $T'$ is the set of subtypes;
$O_n$ is the set of non-baggable objects from $\mathcal{P}$;
$I_l=\phi_\Pi\land m(I_h)$ and $G_l=m(G_h)$.
\end{definition}

Since $\phi_\Pi$ is a part of $I_l$, the mutex groups hold in $I$. 

\begin{theorem} Let $\mathcal{P}$ be an instance of a proper baggable domain. Then 
$\mathcal{B_P}$ is a sound abstraction of $\mathcal{G_P}$.	
\end{theorem}
\begin{proof}
We prove that for each model $M_l$ of   $\mathcal{G_P}$, \ie, an instance of $\mathcal{G_P}$, there is a model $M_h$ of  $\mathcal{B_P}$, \ie, an instance of $\mathcal{B_P}$, s.t. there is an $m-$bisimulation $R$ between $M_h$ and $M_l$. 
First, for the initial situation $S^{M_{l}}_{0}$ of $M_l$, we define the initial situation $S^{M_{h}}_{0}$ of $M_h$ so that $S^{M_{h}}_{0} \sim_{m} S^{M_{l}}_{0}$. Since $S^{M_{l}}_{0}\models m(I_h)$, $S^{M_{h}}_{0}\models I_h$. 
We use the induction method to specify the $m$-bisimulation relation $R$. First, let $\langle S^{M_{h}}_{0}, S^{M_{l}}_{0}\rangle\in R$. As the induction step, if $\langle s_l,s_h\rangle\in R$ and for any HL action $\alpha$, if $s_h$ leads to $s'_h$ via execution of $\alpha$, and $s_l$ leads to $s'_l$ via execution of $m(\alpha)$, then let $\langle s'_h,s'_l\rangle\in R$. By Proposition \ref{tg}, we have $s'_h \sim_{m} s'_l$.
Finally, if $\langle s_h, s_l\rangle \in R$, then $s_h \sim_{m} s_l$.
By Proposition \ref{Giff}, $s_h \models G_h$ iff $s_l\models m(G_h)$.
\end{proof}

\begin{theorem} Let $\mathcal{P}$ be an instance of a proper baggable domain. Then $\mathcal{B_P}$ is a complete abstraction of $\mathcal{G_P}$.
\end{theorem}

\begin{proof}
We prove that for each model $M_h$ of  $\mathcal{B_P}$, there is a model $M_l$ of $\mathcal{G_P}$ s.t. there is a $m$-bisimulation between $M_h$ and $M_l$. We only show how to construct the LL initial situation. The rest of the proof is similar to the soundness proof.  
Let $S^{M_{h}}_{0}$ be the initial situation  of  $M_h$. For each propositional variable $B=p(\vec{o})$, let $S^{M_{l}}_{0} \models p(\vec{o})$ iff $S^{M_{h}}_{0} \models B$. We now describe the objects for each subtype. For each subtype $st$ of type $t$, let $V_{st}=\{N_1,\ldots,N_m\}$. For each $i$, let $S^{M_{h}}_{0}\models N_i=c_i$. We let $st_i$, the subsubtype associated with $N_i$, consist of $c_i$ objects, and we let subtype $st$ be the union of $st_i$. We now define the set of true atoms. Consider $N=(sts(T),eavs(T))$. We choose $t_0 \in T$. Then we have for each $t\in T-\{t_0\}$, $st_N(t_0)$ and $st_N(t)$ have the same size, where $st_N(t_0)$ denotes the subsubtype of $st(t_0)$ associated with $N$. We define a bijection $f_{t_0,t}$ from  $st_N(t_0)$ to $st_N(t)$. Now for each $e\in st_N(t_0)$, we define the tuple $\beta(e)$ as follows: $\beta(e)(t_0)=e$, $\beta(e)(t)=f_{t_0,t}(e)$ for $t\in T_1-\{t_0\}$. Now for each $p(T')\in eavs(T)$, for each $e\in st_N(t_0)$, we let $p$ hold for the projection of $\beta(e)$ to $T'$. If a ground atom is not defined to be true in the above process, then it is defined to be false. We now show that the above definitions for different numeric variables do not interfere with each other. Let $N_1$ and $N_2$ be two different numeric variables. If there is a subtype  $st$ of type $t$ s.t. $N_1,N_2\in V_{st}$, then the above definitions deal with different subsubtypes of $st$; otherwise, the above definitions deal with different types or subtypes. Finally, it is each to prove that for each $N$, if $S^{M_{h}}_{0}\models N=c$, then $S^{M_{l}}_{0}\models m(N)=c$. Thus we have $S^{M_{h}}_{0} \sim_{m} S^{M_{l}}_{0}$.
\end{proof}

Thus, by Theorem \ref{thm1}, a policy $\pi$ solves $\mathcal{B_P}$ iff its refinement solves any instance of  $\mathcal{G_P}$.

\section{Implementation and Experiments}

\begin{table*}[ht]
\small
\centering
\label{tab:Generation}
\scalebox{0.8}{
\begin{tabular}{l|c|c|ccc|cc|ccc|c}
   \hline
    &  & Problem &  \multicolumn{3}{c|}{ Original } & \multicolumn{2}{c|}{ Abstract } &\multicolumn{3}{c|}{ BQNP } &Solving\\
\cline { 4 - 6 } \cline { 7 - 8 } \cline { 9 - 11 }%
   Domain  & $| T^B |$ & Name & $|O_B|/|O_{NB}|$  & $|P_{\downarrow}|/\#facts$ & $| A_{\downarrow}|$ & ABS time(s) & $\#sts$ & $\left|V_N\right|$ & $\left|V_B\right|$ & $|Ops|$ & BQNP time(s) \\
\hline
Gripper-Sim & 2 & prob1-1 & 7/2 & 24/0 & 44 & 0.0110 & 2 & 4 & 2 & 6 & 0.0200 \\
 &  & prob1-2 & 22/2 & 84/0 & 164 & 0.0170 & 2 & 4 & 2 & 6 & 0.0195 \\
 &  & prob2-3 & 25/3 & 168/0 & 609 & 0.0660 & 4 & 13 & 3 & 24 & M \\
\hline
Gripper-HL & 2 & prob1-1 & 10/2 & 40/0 & 70 & 0.0330 & 2 & 6 & 2 & 8 & 16.1030 \\
 &  & prob1-2 & 22/2 & 88/0 & 166 & 0.0380 & 2 & 6 & 2 & 8 & 18.6693 \\
 &  & prob2-1 & 10/2 & 40/0 & 70 & 0.0650 & 3 & 10 & 2 & 13 & M \\
\hline
Gripper-HLWB & 2 & prob1-1 & 10/2 & 56/8 & 70 & 0.1140 & 2 & 10 & 2 & 13 & M \\
 &  & prob2-1 & 10/2 & 56/8 & 70 & 0.1140 & 3 & 10 & 2 & 13 & M \\
\hline
TyreWorld & 1 & prob1-1 & 4/0 & 16/0 & 12 & 0.0110 & 1 & 4 & 0 & 3 & 0.0030 \\
\hline
Ferry & 1 & prob1-1 & 5/2 & 17/0 & 24 & 0.0060 & 1 & 3 & 3 & 6 & 0.0079 \\
 &  & prob2-2 & 5/6 & 41/0 & 96 & 0.0200 & 1 & 7 & 7 & 42 & M \\
\hline
Logistics & 1 & prob1-1 & 4/12 & 68/4 & 328 & 0.0560 & 1 & 10 & 20 & 60 & TO \\
\hline
Transport & 1 & Avg(20) & 6.25/20.70 & 354.80/48.30 & 10790.25 & 1.76 & 5.05 & 82.25 & 42.75 & 1446.40 & TO \\
\hline
Elevators & 1 & Avg(20) & 4.65/15.40 & 525.35/104.40 & 59471.40 & 2.97 & 4.00 & 61.15 & 89.20 & 2799.60 & TO \\
\hline
Floortile & 1 & Avg(20) & 2.40/28.40 & 3331.30/87.00 & 17668.00 & 1.34 & 1.00 & 52.80 & 79.20 & 311.60 & TO \\
\hline
Nomystery & 1 & Avg(20) & 7.50/147.00 & 4495651.25/11862.40 & 283280000.15 & 6.73 & 4.80 & 50.20 & 146.00 & 3654.50 & M \\
\hline
Zenotravel & 1 & Avg(20) & 10.25/18.70 & 445.40/6.00 & 203215.25 & 5.73 & 6.25 & 102.45 & 57.80 & 6164.10 & M  \\
\hline
\end{tabular}
}
\caption{Results of abstracting and solving with time and memory limits of 30m and 8GB. ``TO" and ``M" are used to indicate timeouts and memory out.
$|T^B|$ is the number of baggable types. $|O_B|$ and $|O_{NB}|$ indicate the number of baggable and non-baggable objects. $|P_{\downarrow}|$ is the number of ground atoms. 
$\#sts$ is the number of subtypes.
$|V_N|$ and $|V_B|$ are the number of numerical variables and propositional variables.  $|A_{\downarrow}|$ and $|Ops|$ are the number of ground actions and abstract actions. Avg(n) represents the average results on n problems. ``prob$i$-$j$" with the same $i$ in a domain can be abstracted into the same BQNP problem, where bigger $j$ means bigger problem. 
}
\end{table*}

Based on the proposed abstraction method, we implemented an automated abstraction system ABS with input: the domain and problem description in PDDL format of a planning instance. \citeauthor{HELMERT2009503} (\citeyear{HELMERT2009503}) proposed an algorithm for automatically obtaining mutex groups by focusing on effects of actions. In this paper, we use their system to generate mutex groups whose predicates appear in action effects. For those mutex groups whose predicates do not appear in action effects, we automatically generate them by examining the initial state. 
We implemented a naive BQNP solver BQS using the idea at the end of Section 3, and use the solver to solve the BQNP problems output by ABS. 

All the experiments were conducted on a Windows machine with 2.9GHz Intel 10700 CPU and 16GB memory.

We select $9$ classical planning domains: Gripper, TyreWorld, Ferry, Logistics, Transport, Elevators, Floortile, Nomystery and Zenotravel. For Gripper, we consider 3 versions,  depending on the number of mutex groups. Gripper-Sim is the simplest version, the mutex group of $ball$ is $\{at, carry\}$,  $gripper$ has one $\{carry, free\}$. In Gripper-HL, we add a mutex group $\{HE, LE\}$ for  $gripper$. In Gripper-HLWB, we introduce another mutex group $\{white, black\}$ for $ball$, thus this version is same as in Example \ref{g1}. For all 3 versions, there are two baggable types, and as argued in Section 4.2, the domains are proper. For each of the rest 8 domains, there is only one baggable type, and hence the domain is proper. 
In TyreWorld, the mutex groups of $wheel$ are $\{flat, inflated\}$ and $\{ have, fastened\}$.
In Ferry, the mutex group of $car$ is $\{$\textit{at, on}$\}$.
Logistics is a classic domain, the  mutex group of baggable type $package$ is $\{$\textit{package-at, in-truck, in-airplane}$\}$.
The others are from the IPC competitions.

In our implementation of abstraction, we introduce an optimization trick to remove redundant variables and actions. 
In some domains like Transport, Elevators, there are lots of ground atoms whose truth value remain unchanged. We refer to them as 
 ``facts". We remove these atoms from the HL Boolean variables. Also, we remove numeric variables (HL actions resp.) whose EAVSes  (preconditions resp.)  conflict with the truth values of these atoms. 

As shown in Table 1, for all problems, ABS can produce the BQNP abstractions efficiently (even for problems in Nomystery with $10^8$ ground actions and $10^6$ ground atoms), and the number of abstract actions and variables is significantly less than that of the LL ground actions and atoms (reduced by 50\% to 99\%).  

Our BQNP solver is indispensable for our experimentation, because some BQNP problems BQS can solve cannot be solved with a QNP solver such as DSET or FONDASP. However, BQS suffers from scalability. Solving BQNP with a large state space is beyond the ability of BQS. Especially when a policy involves actions with effects on many variables, BQS quickly reaches the memory limit, 
as observed in examples like Gripper-Sim, Gripper-HL, and Ferry.

\section{Conclusions}

In this paper, we identify a class of STRIPS domains called proper baggable domains, and propose an automatic method to derive a BQNP abstraction from a planning instance of a proper baggable domain. 
Based on Cui et al.'s work, we prove the BQNP abstraction is sound and complete for a g-planning problem whose instances share the same BQNP abstraction with the given instance. Finally, we implemented an automatic abstraction system and a basic BQNP solver, and our experiments on a number of planning domains show promising results. 
We emphasize that our work distinguishes from the work of   on reformulation \cite{Riddle16Baggy,Compiling2016} in that they aim at improving the efficiency of solving a planning instance while we target at solving a g-planning problem induced from a planning instance. Also, our work  distinguishes from the work of  \cite{Illanes19QP} in that they exploit existing work to do the abstraction while we propose a novel abstraction method with both soundness and completeness results. Our research raises the need to improve the scalability of QNP solvers and investigate into QNP solvers for compact policies, and these are our future exploration topics.

\section{Acknowledgments}

We thank the annonymous reviewers for helpful comments. We acknowledge support from the Natural Science Foundation of China under Grant No. 62076261.

\bigskip

\bibliography{aaai25}

\end{document}